\documentclass[11pt,a4paper,oneside,titlepage=true]{amsart} 
\usepackage{microtype}
\usepackage{graphicx}
\usepackage{geometry}
\usepackage[hidelinks]{hyperref}

\usepackage{amsmath,amssymb,amsthm}
\usepackage{balance}
\usepackage{nicefrac}
\usepackage{relsize}

\usepackage{etoolbox}

\def\paragraph#1{\subsubsection*{#1}}

\usepackage[numbers]{natbib}

\let\originalleft\left
\let\originalright\right
\renewcommand{\left}{\mathopen{}\mathclose\bgroup\originalleft}
\renewcommand{\right}{\aftergroup\egroup\originalright}

\usepackage{tikz}
\usepackage{graphicx}
\usepackage{mathtools}
\usepackage{calc}
\usepackage{placeins}

\usepackage[ruled]{algorithm2e}
\SetKwInput{KwData}{Input}
\SetKwInput{KwResult}{Output}

\usepackage[capitalize,noabbrev]{cleveref}

\newtheorem{theorem}{Theorem}[section]
\newtheorem{lemma}[theorem]{Lemma}

\newtheorem{definition}[theorem]{Definition}
\theoremstyle{definition}

\renewcommand\log{\ln}

\newcommand{\vX}{\vec X}

\newcommand{\vk}{\vec k}
\newcommand{\hamming}{\mathrm{d}_H}

\renewcommand{\epsilon}{\eps}

\newcommand\vY{\vec Y}

\newcommand\PSI{\vec\psi}

\renewcommand{\vec}[1]{\boldsymbol{#1}}

\newcommand\SIGMA{\vec\sigma}

\newcommand\G{\mathcal{G}}

\newcommand\cH{\mathcal{H}}

\newcommand\cK{\mathcal{K}}

\newcommand\eps{\varepsilon}

\newcommand\Erw{\mathbb{E}}
\newcommand{\vecone}{\vec{1}}

\newcommand{\set}[1]{\left\{#1\right\}}

\newcommand{\Bin}{{\rm Bin}}

\newcommand{\Be}{{\rm Be}}

\newcommand\bc[1]{\left({#1}\right)}
\newcommand\cbc[1]{\left\{{#1}\right\}}

\newcommand\brk[1]{\left\lbrack{#1}\right\rbrack}

\newcommand\abs[1]{\left|{#1}\right|}

\newcommand\pr{\mathbb{P}} 
\renewcommand\Pr{\pr}

 \def\G{{\vec G}}

\def\pr{{\mathbb P}}

\def\cH{{\mathcal H}}

\newcommand{\remove}[1]{}

\newcommand{\fp}{{\mathcal{P}}}
\newcommand{\fn}{{\mathcal{N}}}

\newcommand{\XI}{\vec{\Xi}}

\usepackage{xspace}

\def\CC{C\nolinebreak[4]\hspace{-.05em}\raisebox{.4ex}{\relsize{-1}{\textbf{++}}}}

\makeatletter
\def\?#1{}
\def\whp{w.h.p\@ifnextchar-{.}{\@ifnextchar.{.\?}{\@ifnextchar,{.}{\@ifnextchar){.}{\@ifnextchar:{.:\?}{.\ }}}}}}
\def\Whp{W.h.p\@ifnextchar.{.\?}{\@ifnextchar,{.}{.\ }}}
\makeatother

\makeatletter
\newcommand{\DeclareMathActive}[2]{%
  \expandafter\edef\csname keep@#1@code\endcsname{\mathchar\the\mathcode`#1 }
  \begingroup\lccode`~=`#1\relax
  \lowercase{\endgroup\def~}{#2}%
  \AtBeginDocument{\mathcode`#1="8000}%
}

\newcommand{\std}[1]{\csname keep@#1@code\endcsname}
\patchcmd{\newmcodes@}{\mathcode`\-\relax}{\std@minuscode\relax}{}{\ddt}
\AtBeginDocument{\edef\std@minuscode{\the\mathcode`-}}
\makeatother

\DeclareMathActive{O}{\operatorname{\std{O}\!}}

\def\dense{}%

\author{Max Hahn-Klimroth, Dominik Kaaser, Malin Rau}
\address{maximilian.hahnklimroth@tu-dortmund.de, TU Dortmund University, Germany}
\address{dominik.kaaser@tuhh.de, TU Hamburg, Germany}
\address{malin.rau@uni-hamburg.de, Universität Hamburg, Germany}
\thanks{The research was supported by the German Research Council, grant DFG FOR 2975.}

\title[Efficient Approximate Recovery from Pooled Data]{Efficient Approximate Recovery from Pooled Data Using Doubly Regular Pooling Schemes}

\begin{document}

\begin{abstract}

In the pooled data problem we are given $n$ agents with hidden state bits, either $0$ or $1$.
The hidden states are unknown and can be seen as the underlying \emph{ground truth} $\sigma$.
To uncover that ground truth, we are given a querying method that queries multiple agents at a time. 
Each query reports the sum of the states of the queried agents.
Our goal is to learn the hidden state bits using as few queries as possible.

So far, most literature deals with \emph{exact reconstruction} of \emph{all} hidden state bits.
We study a more relaxed variant in which we allow a small fraction of agents to be classified incorrectly.
This becomes particularly relevant in the \emph{noisy} variant of the pooled data problem where the queries' results are subject to random noise.
In this setting, we provide a doubly regular test design that assigns agents to queries.
For this design we analyze an approximate reconstruction algorithm that estimates the hidden bits in a greedy fashion.
We give a rigorous analysis of the algorithm's performance, its error probability, and its approximation quality.
As a main technical novelty, our analysis is uniform in the degree of noise and the sparsity of $\sigma$.
Finally, simulations back up our theoretical findings and provide strong empirical evidence that our algorithm works well for realistic sample sizes.
\end{abstract}

\maketitle

\allowdisplaybreaks
\section{Introduction}

In this paper we consider the \emph{pooled data problem} with \emph{additive queries}, defined as follows.
We are given $n$ agents $x_1, \dots, x_n$.
Each agent $x_i$ has a \emph{hidden state bit} $\sigma_i \in \set{0,1}$.
The vector $\SIGMA \in \cbc{0,1}^n$
is called the \emph{ground truth} and is given as a sequence on $n$ independent $\Be(p)-$coin-flips.
The goal is to identify the agents with bit one, i.e., infer the ground truth.
To this end, we can \emph{query} sets of agents:
each query pools a certain number of agents together, measures their hidden states, and returns the sum of the queried agents' states.
Our task is two-fold. 
First, we have to design the \emph{query graph} $G$ that assigns agents to queries.
Second, we have to devise an algorithm that reconstructs the agents' bits given this pooling design and the queries' results.

There is a substantial body of work on the pooled data problem and the related problem of \emph{group testing} in many communities, including statistical physics \cite{alaoui_2017, alaoui_messagepassing}, information theory \cite{AJS_book, tan_2022, coja_spiv}, machine learning \cite{NIPS2014_fb8feff2,kong2012automatic}, and distributed computing \cite{gebhard_2022_ipdps, hahnklimroth_2022_icdcs}. However, all these related works typically consider the problem of \emph{exact reconstruction}. Coming from a more applied direction, we observe that in many practical applications (e.g., in medicine or machine learning) queries are typically subject to (random) errors.
Therefore in this paper, we analyze a relaxation: \emph{approximate} recovery. 
Our goal is to correctly identify \emph{most} of the agents' bits: we allow a fraction of $\epsilon k $ many agents to be classified incorrectly for some small $\epsilon > 0$, where $k = \abs{\abs{\SIGMA}}_1$ is the number of agents with bit one.
As we will see, approximate recovery can be done much more efficiently than exact recovery given realistic sample sizes.

\subsection{Background}
The problem fits well into the so-called
\emph{teacher-student model} introduced by \citet{Gardner_1989}. 
This model provides the fundamental means of our analysis.
The setup is the following: a teacher aims to convey some \emph{ground truth} to a student.
Rather than directly providing the ground truth to the student, the teacher generates observable data from the ground truth via some statistical model and passes both the data and the model to the student.
The student now aims to infer the ground truth from the observed data and the model.

In many real-world scenarios the time to reconstruct the ground truth from the queries' results is dominated by the time to perform a single query.
For instance, in the setting of a life-sciences laboratory queries may be performed by automated pipetting machines that pool samples together and run an automated bio-medical test such as DNA screening \cite{cao_2014, sham_2002} or PCR tests \cite{BENAMI20201248}.
In a technological setting, queries may involve computations  by a neural network on GPUs in a GPU cluster system \cite{liang2021neural, martins_2014, NIPS2014_fb8feff2}.
Finally, various tasks in computational biology \cite{du2000combinatorial,cao_2014,sham_2002}, traffic monitoring \cite{wang_2015}, or confidential data transfer \cite{adam_1989,dinur_2003} use decoding techniques from pooled data.
Since these biological or technological processes are typically quite time-consuming and in order to obtain a substantial speed-up we focus on \emph{non-adaptive} schemes where all queries are specified upfront and executed simultaneously in parallel.

In both the technological setting and the life-sciences setting we observe that queries are subject to random noise.
Therefore we consider the noisy channel model introduced by \citet{hahnklimroth_2022_icdcs} and assume that the agents' hidden bits are subject to random bit flips when read by the queries.
In particular, we assume that with a certain probability $S(1,0)$ a one-bit is read but reported as zero (\emph{false negative}), and with a certain probability $S(0,1)$ a zero-bit is read but reported as one (\emph{false positive}).
Our model also includes the Z-channel with $S(0,1)=0$ where only false negatives occur, i.e., errors of the form $1 \rightarrow 0$.
The Z-channel is particularly important from a practical point of view: it captures the fact that typically in many applications the false-positive probability is insignificant~\cite{constantin1979theory,zhou2013nonuniform}.

\subsection{Related Work}
The pooled data problem finds its roots in the 1970s \cite{djackov_1975} and has since then been frequently studied. 
This reconstruction task of the ground truth is commonly studied from two angles. 
From an \emph{information-theoretic} point of view one asks for the minimum number of queries that are necessary and sufficient such that, given unlimited computational power, inference is possible with high probability. 
From an \emph{algorithmic} point of view one asks for an efficient decoding algorithm. 
In both settings we distinguish between \emph{exact} inference where an algorithm (efficient or not) reconstructs the ground truth completely and \emph{approximate} inference where all but $\eps np$ entries \mbox{are correctly recovered.}

In the literature it is common to distinguish between the \emph{linear regime} ($p = \Theta(1)$) and the \emph{sublinear regime} ($p = n^{-(1 - \theta)}$ for some $\theta \in (0,1)$). 
Typically, techniques for those regimes vary dramatically and the critical phase-transition $\theta \to 1$ is not studied.
Furthermore, many related works \cite{alaoui_2017, alaoui_messagepassing, gebhard_2022_ipdps, hahnklimroth_2022_icdcs, hahnklimroth2021near} assume not only knowledge of $p$ but also of $k$, the exact number of agents with bit one.
In the absence of noise knowing $k$ or knowing $p$ is essentially equivalent: the number of agents with bit one is strongly concentrated around the mean, and a single global query outputs the exact value of $k$. 
Under noise, however, things become much more complicated.

\citet{djackov_1975} provided an information-theoretic lower bound $m_{PD}$ which states that, even when given unlimited computational power, no algorithm can infer the ground truth with less than
\begin{align*}
    m_{PD} \sim 2 n \cH \bc{\frac{k}{n}} \frac{ \log \frac{n}{k} }{ \log k}
\end{align*}
queries. Here, $\cH(\alpha) = - \alpha \log(\alpha) - (1-\alpha) \log(1 - \alpha)$ denotes the entropy.
In the early 2000s, \citet{grebinski_2000} proved that exponential-time constructions can solve the pooled data problem with no more than $2 m_{PD}$ queries.
More recently it was established by \citet{scarlett_2017} (for the case $k = \Theta(n)$) and by \citet{gebhard_2022_ipdps, feige2020quantitative} (for the case $k \sim n^{\theta}$) that $m_{PD}$ queries suffice from an information-theoretic point of view.

Less is known from the algorithmic perspective.
For a long time, only algorithms that exceed $m_{PD}$ by a factor of $\Omega(\log n)$ were known:
in the linear regime message-passing algorithms were proposed, and non-rigorous ideas from statistical physics suggest that those algorithms require $\Theta(n)$ queries \cite{alaoui_messagepassing}. 
In the sublinear regime ideas from coding theory \cite{karimi_2019_2, karimi_2019}, thresholding algorithms \cite{gebhard_2022_ipdps, hahnklimroth_2022_icdcs} and algorithms for group testing \cite{AJS_book} were studied, all of which require $O \bc{ k \log n }$ queries. 
Recently an efficient algorithm that requires $O\bc{k}$ queries was proposed \cite{hahnklimroth2021near}, closing the $\log n$-gap in the sublinear regime in the absence of noise. While this algorithm is efficient from a theoretical point of view, the proofs and techniques used in the construction require tremendously large values of $n$. In the linear regime this gap is still an open question \cite{feige2020quantitative}.

Most contributions use an almost regular random graph as a pooling design.
This means that either queries select agents uniformly at random, or agents select tests uniformly at random or, sometimes, the mapping is based on independent coin-flips. 
For large values of $n$, all those approaches are expected to become equivalent on dense pooling designs due to strong concentration properties.
However, on small instances, fluctuations in the number of queries per agent or the number of agents per query might contribute to the performance of algorithms. 
Works that explicitly focus on small values of $n$ are only known in the related \emph{group testing} problem. 
In group testing, a query returns the information whether at least one agent in the query has bit one. 
The findings suggest that almost-regular designs perform best if the fluctuations are small and queries are not too similar \cite{tan_2022,Noonan_2021}.

To the best of our knowledge, all rigorous algorithmic results with respect to the pooled data problem concern \emph{exact} recovery. 
Rigorous results with respect to information-theoretic aspects are studied by \citet{scarlett_2017}, and heuristics from statistical physics suggest that approximate recovery is as hard as exact recovery \cite{alaoui_messagepassing}. 
Nevertheless, in the group testing problem, the class of low-degree polynomials (including important classes of algorithms such as local algorithms) are understood with respect to approximate recovery~\cite{aco_low_degree_gt_colt}.

\subsection{Our contribution}
We study an algorithm based on the thresholding approach by \citet{gebhard_2022_ipdps, hahnklimroth_2022_icdcs} for the approximate recovery problem under noise and for a broad range of prior probabilities $p$.
More precisely, we only require $p \gg n^{-1} \log n$, and we explicitly do not require knowledge of $k$, the number of agents with bit one. 
We give an explicit bound on the number of queries required for the algorithm to recover all but $\eps np$ agents correctly with probability at least $1 - \delta - o(1)$. 
To this end we modify the pooling design proposed by \citet{gebhard_2022_ipdps, hahnklimroth_2022_icdcs} such that all agents are in the same number of queries and each query contains the same number of agents (up to $\pm 1$ due to divisibility conditions). 
This doubly regular design is harder to study since it introduces stochastic dependencies. However, as our simulations show, the design is superior when the number of agents is small and the queries are subject to noise.

\paragraph{Outline.}
In the next section we formally introduce our pooling design, present our reconstruction algorithm, and state our main theorem. Then in \cref{sec:simulations} we present implementation details and empirical results of our simulations.
The formal analysis of our main result is presented in \cref{sec:analysis}.
Finally, in \cref{sec:discussion} we discuss our findings and conclusions, and we present some open problems.

\section{Model and Contribution} \label{sec:model}

Let $x_1, \ldots, x_n$ denote the $n$ agents such that any agent has bit one with probability $p$ independently from each other.
We let $\vec k$ denote the (random) number of agents with hidden bit one, thus $\Erw \brk{\vec k} = np$. 
We assume in the following that $p = \omega \bc{n^{-1} \log n}$.

Recall that the pooling design is defined via a bipartite \mbox{(multi-)} graph that assigns agents to queries.
In our pooling design we assume that the number of agents per query $\Gamma$ is fixed, and we define the following pooling design $\G_\Gamma$.
Given the number of queries $m$ and the number of agents per query $\Gamma$ with $ n^{\delta} \sqrt{n (mp)^{-1}} \leq \Gamma \leq n^{1 - \delta}, \delta > 0, $ we set the sequence of agent-degrees $\Delta_1, \ldots, \Delta_n$ such that 
\[
\abs{ \Delta_i - \Delta_j } \leq 1 \quad \text{and} \quad \smash{\sum_{i=1}^n} \Delta_i = m \Gamma.    
\]
With these agent degrees we let $\G = \G_\Gamma$ be a  bipartite multi-graph chosen uniformly at random chosen according to the configuration model (see, for instance \cite{Janson2000}). An example of such a pooling design can be seen in \cref{fig:figure_pooling}. Given $\G$, let $\Delta_i^\star$ denote the number of distinct queries agent $i$ is part of. 
Furthermore, let $\Delta = n^{-1} \sum_{i=1}^n \Delta_i$ and $\Delta^\star = n^{-1} \sum_{i=1}^n \Delta_i^\star$ denote the average agent-degrees. 

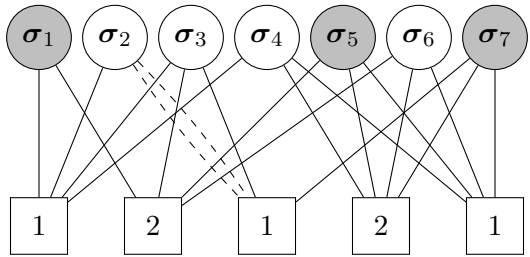
\begin{figure}[t]
\centering
\begin{tikzpicture}[scale=1]
\node[circle, draw, minimum width=0.75cm, fill=black!25] (x0) at (0, 0) {$\SIGMA_1$};
\node[circle, draw, minimum width=0.75cm] (x1) at (1,0) {$\SIGMA_2$};
\node[circle, draw, minimum width=0.75cm] (x2) at (2, 0) {$\SIGMA_3$};
\node[circle, draw, minimum width=0.75cm] (x3) at (3, 0) {$\SIGMA_4$};
\node[circle, draw, minimum width=0.75cm, fill=black!25] (x4) at (4, 0) {$\SIGMA_5$}; 
\node[circle, draw, minimum width=0.75cm] (x5) at (5, 0) {$\SIGMA_6$};
\node[circle, draw, minimum width=0.75cm, fill=black!25] (x6) at (6, 0) {$\SIGMA_7$};

\node[rectangle, draw, minimum width=0.75cm, minimum height=0.75cm] (a1) at (0, -2.5) {$1$};
\node[rectangle, draw, minimum width=0.75cm, minimum height=0.75cm] (a2) at (1.5, -2.5) {$2$};
\node[rectangle, draw, minimum width=0.75cm, minimum height=0.75cm] (a3) at (3, -2.5) {$1$};
\node[rectangle, draw, minimum width=0.75cm, minimum height=0.75cm] (a4) at (4.5, -2.5) {$2$};
\node[rectangle, draw, minimum width=0.75cm, minimum height=0.75cm] (a5) at (6, -2.5) {$1$};

\path[draw] (x0) -- (a1);
\path[draw] (x0) -- (a2);
\path[draw] (x1) -- (a1);
\path[-] (x1) edge [bend left=5,dashed] (a3);
\path[-] (x1) edge [bend right=5,dashed] (a3);
\path[draw] (x2) -- (a1);
\path[draw] (x2) -- (a3);
\path[draw] (x2) -- (a2);
\path[draw] (x3) -- (a1);
\path[draw] (x3) -- (a4);
\path[draw] (x3) -- (a5);
\path[draw] (x4) -- (a2);
\path[draw] (x4) -- (a5);
\path[draw] (x4) -- (a4);
\path[draw] (x5) -- (a2);
\path[draw] (x5) -- (a4);
\path[draw] (x5) -- (a5);
\path[draw] (x6) -- (a3);
\path[draw] (x6) -- (a4);
\path[draw] (x6) -- (a5);
\end{tikzpicture}
\caption{A small example with $n = 7$ agents and ground truth $\SIGMA = (1,0,0,0,1,0,1)$. A gray vertex color indicates that the agent has bit one. The underlying multi-graph defines the queries. Ultimately, the goal is to reconstruct $\SIGMA$ as well as possible given $\G$ and the queries' results. }
\label{fig:figure_pooling}
\end{figure}
\begin{algorithm}[t]
\small
\DontPrintSemicolon
\SetAlgoVlined
\SetKwFor{ForParallel}{for}{do in parallel}{end}
\SetKwFor{ForAt}{at}{do}{end}
\SetKw{KwTo}{to}
\SetKw{KwWith}{with}
\SetKwFunction{Query}{query}
\SetKwProg{Initially}{I.~Perform Measurements}{}{end}
\SetKwProg{Reconstruct}{II.~Output Estimate}{}{end}
\Initially{}{
add $m$ query nodes to the network\;
\ForAt{query node $a_j$}{
    sample agents $\cbc{v_1, \dots, v_\Gamma}$ u.a.r.\ with replacement\;
    measure $\hat \SIGMA_j \simeq \sum_{i = 1}^{\Gamma} \SIGMA(v_i)~$ 
    \tcp{\parbox{0pt}{\parbox{5cm}{$\hat \SIGMA_j$ {\normalfont{}is subject to noise!}}}}
    send message $\hat \SIGMA_j$ to each distinct neighbor\; 
}
\ForAt{agent $x_i$ \KwWith incoming message $\hat \SIGMA_j$}{
    update score $\Psi_i \gets \Psi_i + \hat \SIGMA_j$\;
}
}
\Reconstruct{}{
    \ForAt{agent $x_i$ \KwWith score $\Psi_i$}{
    \parbox{2.5cm}{\textbf{if} $\Psi_i \leq T$ \textbf{then}} agent $x_i$ outputs $0$\;
    \parbox{2.5cm}{\textbf{else} } agent $x_i$ outputs $1$\;
}
}
\caption{Reconstruction Algorithm with threshold $T$}
\label{algo_mn}
\end{algorithm}

\subsection{Thresholding Algorithm}
We study a distributed combinatorial reconstruction algorithm that makes use of the following observation: if an agent has hidden bit one it increases the queries' results every time it is queried.
The expected sum of an agent's queries' results is therefore larger if the agent has bit one.
The expected difference in this \emph{score} between agents with bit one and agents with bit zero becomes larger if more queries are conducted. 
Therefore, the algorithm uses a threshold value $T$ and declares all agents with a score larger than $T$ as having bit one.  
This algorithm was originally introduced to study the noiseless variant of the pooled data problem for exact recovery \cite{gebhard_2022_ipdps} and was later studied for the purpose of exact recovery under noise~\cite{hahnklimroth_2022_icdcs}.
It is formally specified in \cref{algo_mn}.

\subsection{Formal Results}
Our main theorem includes results for a general noise model, the noisy channel model. It is formally defined as follows.
\begin{definition}[Noisy-Channel]
Under the noisy channel a query reads the bit of an agent according to the following noise channel with probability given by $S$, where
\begin{align*}
    S(i,j) = \Pr \bc{ \text{bit } j \text{ is read } \mid \text{bit } i \text{ was sent } }.
\end{align*}
We assume that all entries of $S$ are real values not dependent on $n$ with $S(1,1) - S(0,1) > 0$.
\end{definition}
The last assumption simply says that reading a single bit as one is more likely if the true signal was present as well. 
From here on, we let $\G$ be an instance of the almost $(\Gamma, \Delta)$-doubly regular pooling design.
Next, we properly define what is meant by \emph{approximate reconstruction}.
Fix a failure probability $\delta$ and an approximation fraction $\eps > 0$.
\begin{definition}[$\eps$-recovery]
An algorithm is said to be an $\eps$-recovery algorithm if, given the pooling design $\G$ and the queries' results $\hat \SIGMA$, it outputs an estimate $\tilde \sigma$ such that $\hamming(\SIGMA, \tilde \sigma) \leq 2 \eps \abs{\abs{\SIGMA}}_1$.
\end{definition}
Observe that, if $\abs{\abs{\SIGMA}}_1 = \abs{\abs{\tilde \sigma}}_1$, an $\eps$-recovery algorithm declares at most $\eps \vk = \eps \abs{\abs{\SIGMA}}_1$ agents with bit one to have bit zero and vice versa.
Our main result states an upper bound on the required number of queries such that \cref{algo_mn} achieves $\eps$-recovery with failure probability $\delta + o_n(1)$. Here, $o_n(1)$ denotes a function that tends to \mbox{zero with $n \to \infty$.}
\begin{theorem} 
\label{thm_mainresult}
Let $S$ be the channel matrix and let \def\lefttag#1{\tag*{\makebox[0pt][l]{\hspace*{-\linewidth}\textit{#1}}}}
\begin{align*}
L = L(n, p, S) =\frac{\bc{ S(1,1) - S(0,1) }^2}{ 2 n \bc{ S(0,1) + p (S(1,1) - S(0,1)) }}
\intertext{and}
 \smash[t]{T^\star = \Delta S(1,1) - \bc{\frac{1}{2} - \frac{\log \bc{\frac{1}{p}}}{2 L m} } \Delta (S(1,1) - S(0,1)).}
\end{align*}
Then \cref{algo_mn} with threshold $T^\star$ is with probability $1 - \delta - o_n(1)$ an $\eps$-recovery algorithm on the almost $(\Gamma, \Delta)$-doubly regular pooling design if the number of queries $m$ is at least 
\begin{align*}
    m\!>\! 
    \frac{1}{L}\left(\ln\left(\frac{1}{p}\right) + 2\ln\left(\frac{2}{\eps\delta}\right) + 2\sqrt{\ln\left(\frac{2}{\eps\delta}\right)\ln\left(\frac{2}{\eps\delta p}\right)}\right).
\end{align*}

\end{theorem}
While our main theorem is an asymptotic result (due to the $o_n(1)$ in the failure probability), the next section analyses the performance of \cref{algo_mn} for a decent number of agents empirically. 

\section{Simulations}\label{sec:simulations}

\begin{figure}[ht]
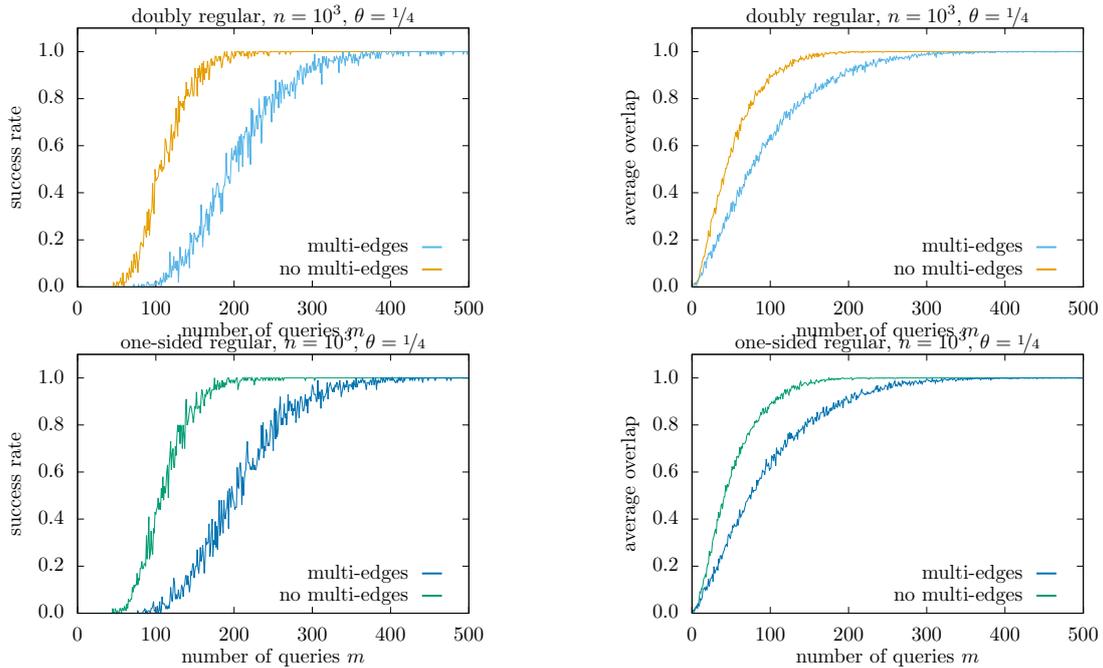
\small
\resizebox{0.45 \linewidth}{!}{
\input{figures/comparison-multi-edges-1000-0.25-success}
}
\hfill
\resizebox{0.45 \linewidth}{!}{
\input{figures/comparison-multi-edges-1000-0.25-overlap}
}
\resizebox{0.45 \linewidth}{!}{
\input{figures/comparison-multi-edges-II-1000-0.25-success}
}
\hfill
\resizebox{0.45 \linewidth}{!}{
\input{figures/comparison-multi-edges-II-1000-0.25-overlap}
}
\caption{Comparison of the algorithm's performance on designs with multi-edges and the corresponding designs without multi-edges. The left plots show the fraction of trials in which at least 90 \% of the agents with hidden bit one were learned correctly. The right plots show the average overlap, i.e., the average number of correctly classified agents with bit one. All simulations are conducted without noise.}
\label{fig:multiedges}
\end{figure}
\begin{figure}[ht]
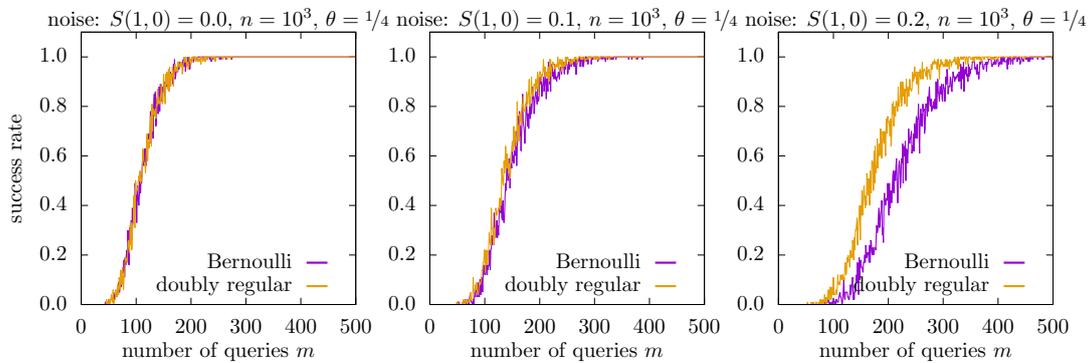
\small
\resizebox{0.32 \linewidth}{!}{
\hfill\input{figures/noise-1000-0.25-0.0-success.tex} }
\hspace{-1em}
\resizebox{0.32 \linewidth}{!}{
\input{figures/noise-1000-0.25-0.1-success.tex} }
\hspace{-1em}
\resizebox{0.32 \linewidth}{!}{
\input{figures/noise-1000-0.25-0.2-success.tex} }
\caption{Comparison of the algorithm's performance on designs with a different degree of regularity. If no noise is present, the Bernoulli model and the doubly regular model perform equally well. Under noise, however, the additional regularity on the agents improves the performance of the algorithm visibly.}
\label{fig:regularity}
\end{figure}

There are essentially three closely related strategies to generate a random graph with expected number of queries per agent  $\Delta$ and expected number of agents per query $\Gamma$:
\begin{itemize}
    \item \emph{Bernoulli}: for any query and any agent, a random coin is flipped whether the agent is part of the query.
    \item \emph{one-sided regular}: every query chooses exactly $\Gamma$ agents uniformly at random.
    \item \emph{doubly regular}: establish a random mapping between agents and queries such that any agent is in $\Delta \pm 1$ queries and every query contains $\Gamma$ agents.
\end{itemize}
Both regular designs can be defined in a way such that an agent appears multiple times in a query (the underlying pooling graph has multi-edges) and in a way such that the underlying pooling graph is a simple graph.
Our formal results are in line with recent literature and allow agents to be queried multiple times.
From a theoretical, asymptotic point of view, it is indeed known that our algorithm performs better when multi-edges are allowed \cite{gebhard_2022_ipdps,hahnklimroth_2022_icdcs}.
However, for realistic sizes of $n$, it is far from clear whether agents appearing multiple times \mbox{in queries help.}

We implemented our simulations in the \CC{} programming language.
As a source of randomness, we use the Mersenne Twister 
\texttt{mt19937\_64} provided by the \CC{11} \texttt{\textless random\textgreater{}} library.
In a simulation run, we first generate the random pooling scheme by constructing a random bipartite graph according to one of the three strategies described above.
Then we perform a faithful simulation of the distributed system where we compute the agents' scores as defined in \cref{algo_mn}.
Note that for the sake of reproducibility and comparability we did not use a different, random number of agents with bit one for each simulation run.
Instead we set $k$ to either $n^{\nicefrac{1}{8}}$, 
$n^{\nicefrac{1}{4}}$, or $n^{\nicefrac{1}{2}}$.
In the following, we present data for $n = 10^3$ agents.
Each data point stems from $100$ independent simulation runs.
We show the success rate, defined as the fraction of simulations in which more than 90\% of agents with hidden bit one were classified correctly, and the average overlap, defined as the overall fraction of correctly classified agents with hidden bit one.
The simulation software and all necessary tools to reproduce our plots will be made available on our public GitHub repository.

The first question of our empirical analysis is whether the appearance of multi-edges is beneficial.
Our simulations reveal that this is not the case.
Adding insult to injury, our empirical data show that multi-edges in the pooling graph are actually harmful, see  \cref{fig:multiedges}.

The second question of interest in the influence of the regularity of the pooling scheme.
In group testing, the more regular the design is, the fewer queries are required \cite{tan_2022,Noonan_2021}.
The next set of results studies this effect in the pooled data problem.
Interestingly, the degree of regularity has only a limited impact on the performance if all queries report perfect information: in the noiseless setting, the algorithm performs equally well on the three model variants.
In contrast, if queries are subject to noise, the same effect as in group testing is visible: doubly regular designs dramatically outperform non-regular models, see \cref{fig:regularity}.

\section{Formal Analysis} \label{sec:analysis}
In this section, we first prove some basic properties of the pooling scheme and pin down the distribution of the score used in the algorithm. 
Afterwards, we analyze how an agent's score depends on the agent's hidden bit.
It turns out that the scores between agents of both types follow a multivariate hypergeometric distribution and the corresponding expected values differ.
We finally optimize a threshold such that due to concentration properties of the hypergeometric distribution, most scores below the threshold belong to agents with bit zero and most scores above the threshold belong to agents with bit one.

\subsection{Properties of the Pooling Scheme} \label{sec:properties-pooling-scheme}
The first lemma shows that although there is a substantial number of multi-edges in the pooling graph, any individual agent appears in almost only distinct queries. 
\begin{lemma}
\label{lem_conc_deltastar}
If $\Gamma = o(n)$, we have with probability $1-o(1)$ that for all $1 \leq i \leq n$,
\begin{align*}
    \Delta_i^\star = (1 + o(1)) \Delta .
\end{align*}
\end{lemma}
\begin{proof}
The probability that a specific agent $x_i$ is part of a specific query $a_j$ is given by
\begin{align*}
    p_c &= 1 - \frac{ \binom{\sum_j \Delta_j - \Delta_i}{\Gamma} }{ \binom{\sum_{j} \Delta_j}{\Gamma} } = (1 + o(1)) \bc{ 1 - \frac{ \binom{n \Delta - \Delta}{\Gamma} }{ \binom{n \Delta}{\Gamma} } }.
\intertext{A short calculation that uses the equality $n \Delta = m \Gamma$ verifies}
    p_c &= (1 + o(1)) \bc{ 1 - \prod_{j=1}^\Delta \bc{ 1 - \frac{\Gamma - j}{ n\Delta - j } }} \\  
    &= (1 + o(1)) \bc{1 - \bc{1 - \frac{1}{m}}^{\frac{m \Gamma}{n}} }\\ & = (1 + o(1)) \bc{ 1 - \exp \bc{ - \frac{\Gamma}{n} } }.
\end{align*}
As the events of being part of a specific query are negatively associated under the configuration model, the Chernoff bound establishes that $\Delta_i^\star = (1 + o(1)) m p_c$ \cite{chen2013concentration}. 
On the other hand, $\Delta = \frac{m \Gamma}{n}$ by construction. 
If $\Gamma = o(n)$ as supposed, we have by a Taylor expansion that $p_c = (1 + o(1)) \frac{\Gamma}{n}$, thus $\Delta_i^\star = (1 + o(1)) \frac{m \Gamma}{n}$ and the lemma follows. 
\end{proof}

Next, we define the agents' \emph{scores} $\PSI_i$. In the end, the score will be used to decide whether an agent's hidden bit is zero or one.
For an agent $x_i$, given the pooling graph $\G$ and the sequence of queries' results $\hat \SIGMA$, let  $$\PSI_i = \sum_{j = 1}^m \vecone \cbc{ x_i \in a_j } \hat \SIGMA_j$$ be the score of $x_i.$ 
Clearly, this score is expected to be larger by the additive number $\Delta_i^\star$ if $\SIGMA_i = 1$. 
In previous works, the pooling design had no restrictions on the agent's degree which induced a lot of stochastic independence on $\PSI$. 
Moreover, the designs were not built on the configuration model and it was supposed that the exact number of agents with bit one was known a priori. 
In the current contribution, we need to be much more careful. 
Throughout this section, we implicitly condition on the $\sigma-$algebra generated by the edges of the underlying random graph.

The next lemma describes the distribution of the agents' scores under perfect queries. 
While in the actual model we assume to queries to be noisy, this case will help to understand the (much more involved) general distribution.
\begin{lemma}[Score distribution without noise] \label{lem_scoredist}
If the channel matrix $S$ is the identity matrix (no noise), we find that
$$ \PSI_i \stackrel{d}{=} \vec X_i + \Delta_i \vecone \cbc{ \SIGMA_i = 1 }$$
where $\vec X_i$ is distributed as follows.
Given $\SIGMA$, it is hypergeometrically distributed with population size $ \sum_{ j \neq i}\Delta_j$, $\sum_{ j \neq i} \vecone{ \cbc{ \SIGMA_j = 1} } \Delta_j$ successes and $\Gamma \Delta_i^\star - \Delta_i$ draws.
\end{lemma}
\begin{proof}
The score $\PSI_i$ consists of two parts. 
First, the contribution of $x_i$. 
Second, the contribution of all agents in the second neighborhood of $x_i$.
Observe that, due to the appearance of multi-edges, $x_i$ can be part of this second neighborhood itself.
We need to calculate the distribution of the number of agents with bit one connected to $\partial_{\G}(x_i)$ in the subgraph $\G - x_i$. 
By definition of the configuration model, we find that in this graph, the queries need to choose exactly  $\Gamma \Delta_i^\star - \Delta_i$ distinct half-edges. 
The score is increased by one if and only if a half-edge is connected to an agent with bit one under the ground truth. 
There are exactly $\sum_{ j \neq i} \vecone{ \cbc{ \SIGMA_j = 1} } \Delta_j$ half-edges with this property while in total, there remain $\sum_{ j \neq i}\Delta_j$ half-edges in $\G - x_i$. 
Finally, if $\SIGMA_i = 1$, thus $x_i$ has bit one, then the score $\PSI_i$ is increased by $\Delta_i$.
\end{proof}

For the sake of intuition, we remark that under the knowledge of $\vk$, 
\begin{align*}
    \sum_{ j \neq i}\Delta_j &\approx (n-1) \Delta\\ \quad \sum_{ j \neq i} \vecone{ \cbc{ \SIGMA_j = 1} } \Delta_j &\approx (\vec k - \vecone{\cbc{\SIGMA_i = 1}}) \Delta
\end{align*}
and so $\vX_i$ has approximately as good concentration properties as a $$\Bin\bc{\Delta^\star \Gamma-\Delta,\frac{ \vk-\vecone\{ \SIGMA(j) = 1\} }{n-1}}$$ random variable.

The next lemma generalizes the distribution of the agents' scores to the more general model in which queries are exposed to noise.
Observe that for $S$ being the identity matrix, this boils down to the previous lemma.

\begin{lemma}[Score distribution under noise] \label{cor_scoredist_2}
Given the channel matrix $S$ and $\vk$, let $ N_i = \sum_{ j \neq i}\Delta_j$, $K_i = \sum_{ j \neq i} \vecone{ \cbc{ \SIGMA_j = 1} } \Delta_j$ and $n_i = \Gamma \Delta_i^\star - \Delta_i$. 
Let $\vY_i$ be a multivariate hypergeometrically distributed random variable on four types in a population of size $N_i$ with $n_i$ draws. 
We let $K_i(1,1) = K_i S(1,1)$, $K_i(0,1) = (N_i - K_i) S(0,1)$, $K_i(1,0) = K_i S(1,0)$ and $K_i(0,0) = (N_i - K_i) S(0,0)$ be the number of half-edges in the population of the four types.
We have
\begin{equation*}
\begin{split}
   \PSI_i  \stackrel{d}{=} \vY_i(1,1) + \vY_i(0,1)    &+  \Bin \bc{\Delta_i, S(1,1)} \vecone \cbc{ \SIGMA_i = 1 }\\ &  + \Bin \bc{\Delta_i, S(0,1)} \vecone \cbc{ \SIGMA_i = 0 } .
\end{split}
\end{equation*}
Furthermore, we have that $\PSI_i$ is a sum of negatively associated Bernoulli random variables.
\end{lemma}
\begin{proof}
As a first step, we prove that
\begin{align}
   \label{cor_scoredist}\PSI_i  \stackrel{d}{=} & \Bin \bc{ \vec X_i, S(1,1) } + \Bin \bc{ N_i - \vec X_i, S(0,1) }   \\
   & \quad + \Bin \bc{\Delta_i, S(1,1)} \vecone \cbc{ \SIGMA_i = 1 }  + \Bin \bc{\Delta_i, S(0,1)} \vecone \cbc{ \SIGMA_i = 0 }. \notag
\end{align}
For any of the half-edges connected to agents in $\G - x_i$, any of the $\vX_i$ one-bits is transmitted correctly independently from everything else with probability $S(1,1)$, so given $\vX_i$, the number of such bits is binomially distributed. 
Analogously, any of the $N_i - \vX_i$ zero-bits is transmitted falsely with probability $S(0,1)$. 
Finally, for any of the $\Delta_i$ edges connected to $x_i$, the same channel argument is applied as the transmission is independent from the rest. This proves \eqref{cor_scoredist}.

The lemma now follows from \eqref{cor_scoredist} and the following combinatorial insight. 
Given $N$ balls out of which $K$ are red and $N-K$ are blue, some red balls will be colored lime and some blue balls will be colored orange. 
In the end, we want to describe the distribution of the sum of the red and orange balls in a sample. 
The two following random experiments yield the same distribution of the sum of orange and red balls.
\begin{itemize}
    \item  First random experiment
    \begin{itemize}
        \item Draw $n$ balls without replacement out of the $N$ balls, let $S$ be the sample.
        \item For any red ball in $S$, color it independently from anything else with probability $p$ in lime.
        \item For any blue ball in $S$, color it independently from anything else with probability $q$ in orange.
    \end{itemize}
    \item  Second random experiment
    \begin{itemize}
        \item For any red ball out of the $N$ balls, color it independently from anything else with probability $p$ in lime.
        \item For any blue ball out of the $N$ balls, color it independently from anything else with probability $q$ in orange.
        \item Draw $n$ balls without replacement out of the $N$ balls.
    \end{itemize}
\end{itemize}
The distribution of \cref{cor_scoredist} is given by the first experiment while the distribution described in \cref{cor_scoredist_2} is given by the second experiment. 
This concludes the distributional equality.

The second statement, namely the negative association, is a well-known property of the (multivariate) hypergeometric distribution \cite{Joag_1983}.
\end{proof}

\subsection{Differences Based on the Hidden Bit}

Finally, we need to make sure that at most $2 \eps \vk$ agents are classified incorrectly. 
We will order the agents by their centralized scores $ \PSI_i - \Erw \brk{ \vX_i }$ and declare any agent with a score above a certain threshold $T$ as having bit one and to have bit zero otherwise. 
We only need to conduct enough queries such that the concentration properties of $\PSI_i - \Erw \brk{ \vX_i }$ are good enough such that only $2 \eps \vk$ errors occur with probability $1 - \delta$. 

It is easily verified that
\begin{equation}
\begin{aligned}
    & \Erw \brk{ \PSI_i \mid \SIGMA } \\
    &= n_i \frac{K_i}{N_i}S(1,1) + n_i \frac{N_i - K_i}{N_i}S(0,1) \label{eq_expectation_psi}\\
    & \phantom{{}={}} + \Delta_i S(1,1) \vecone \cbc{\SIGMA_i = 1} + \Delta_i S(0,1) \vecone \cbc{\SIGMA_i = 0} \\
    & = \bc{\Gamma \Delta_i^\star - \Delta_i} \\ &\phantom{{}={}}  \cdot \bc{\frac{ \vk }{n} \bc{S(1,1) - S(0,1)} + S(0,1) + O \bc{ n^{-1} } } \\
    & \phantom{{}={}} + \Delta_i S(1,1) \vecone \cbc{\SIGMA_i = 1} + \Delta_i S(0,1) \vecone \cbc{\SIGMA_i = 0} \end{aligned}
\end{equation}
and thus { 
\begin{align*}
    \Erw \brk{ \PSI_i \mid \SIGMA_i \!=\! 1 } - \Erw \brk{ \PSI_i \mid \SIGMA_i \!=\! 0 } &= \Delta_i \bc{ S(1,1) - S(0,1) }.
\end{align*}}
Observe that this difference depends only on the hidden bit of agent $x_i$ and its degree in the underlying graph. It does explicitly not depend on the prior $p$ nor the number of positive hidden bits $\vec k$.

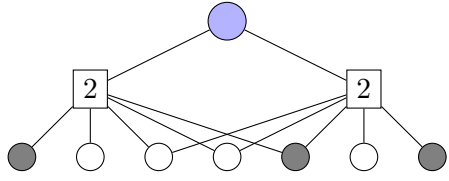
\begin{figure}[t]
\centering
\begin{tikzpicture}[scale=0.9]
\node[circle, draw, minimum width=0.5cm, fill=blue!30] (x0) at (4, 2) {};
\node[circle, draw, minimum width=0.3cm, fill=gray] (x1) at (1,0) {};
\node[circle, draw, minimum width=0.3cm, fill=white] (x2) at (2, 0) {};
\node[circle, draw, minimum width=0.3cm, fill=white] (x3) at (3, 0) {};
\node[circle, draw, minimum width=0.3cm, fill=white] (x4) at (4, 0) {}; 
\node[circle, draw, minimum width=0.3cm, fill=gray] (x5) at (5, 0) {};
\node[circle, draw, minimum width=0.3cm, fill=white] (x6) at (6, 0) {};
\node[circle, draw, minimum width=0.3cm, fill=gray] (x7) at (7, 0) {};

\node[rectangle, draw, minimum width=0.4cm, minimum height=0.4cm] (a1) at (2, 1) {$2$};
\node[rectangle, draw, minimum width=0.4cm, minimum height=0.4cm] (a2) at (6, 1) {$2$};

\path[draw] (a1) -- (x0);
\path[draw] (a2) -- (x0);

\path[draw] (a1) -- (x1);
\path[draw] (a1) -- (x2);
\path[draw] (a1) -- (x3);
\path[draw] (a1) -- (x4);
\path[draw] (a1) -- (x5);

\path[draw] (a2) -- (x3);
\path[draw] (a2) -- (x4);
\path[draw] (a2) -- (x5);
\path[draw] (a2) -- (x6);
\path[draw] (a2) -- (x7);
\end{tikzpicture}
\label{fig_second_neighborhood}
\caption{Example of a second neighborhood of the blue agent in the pooling graph. The distribution of the number of agents with bit one is independent from the agent's bit. }
\end{figure}

We, therefore, need to establish conditions, such that for all but $\eps \vk~$ agents $x_i$ with bit one and all but $\eps \vk~$agents $x_j$ with bit zero the contribution \[ \abs{ \bc{\vY_i(1,1) + \vY_i(0,1)} - \bc{  \vY_j(1,1) + \vY_j(0,1) } } \]  is, with probability at least $1- \delta$, smaller than $\Delta \bc{ S(1,1) - S(0,1) } + O(1)$ (see \cref{cor_scoredist_2}). 
In the last step, we used that $\abs{\Delta_i - \Delta} \leq 1$ by definition of the pooling design.

We will give an alternative sufficient condition, namely that for all agents $x_i$, we find that $\bc{\vY_i(1,1) + \vY_i(0,1)}$ is sufficiently well concentrated around its mean.
If we denote \[p_S = p \bc{ S(1,1) - S(0,1) } + S(0,1),\] then as in \eqref{eq_expectation_psi},
\begin{align*}
    \Erw \, \brk{ \vY_i(1,1) + \vY_i(0,1) } = \Gamma \Delta^\star \bc{ p_S + O \bc{ n^{-1} } }.
\end{align*}

\begin{lemma} \label{lem_chernoff_beta}
Let $\XI_i = \vY_i(1,1) + \vY_i(0,1)$. 
For all $\beta > 0$, we find that
\begin{align*}
    &\Pr \bc{ \abs{\XI_i - \Erw \brk{ \XI_i } } > \beta \Delta   } \leq  2 \exp \bc{ - \frac{ (1 + o(1)) \beta^2 m}{ 2 n p_S  } }
\end{align*}
\end{lemma}
\begin{proof}
By a standard application of the Chernoff bound for the hypergeometric distribution \cite{Janson99onconcentration} and \eqref{eq_expectation_psi} we have
\begin{align*}
    \Pr &\bc{ \abs{ \XI_i - \Erw \brk{ \XI_i \mid \vk } } > \beta \Delta  \mid \vk }\\ 
    &\leq 2 \exp \bc{ - (1 + o(1) \frac{\beta^2 \Delta^2}{ (2 + \beta / (\Gamma p_S)) \Erw \brk{ \XI_i \mid \vk } } } \\
\end{align*}
Because $\Erw \brk{ \XI_i \mid \vk } = \Delta^\star \Gamma p_S$, and $\beta / (\Gamma p_S) = o(1)$ for any fixed $\beta \in (0,1)$, by \cref{lem_conc_deltastar}, the fact that $\Gamma = \frac{n \Delta}{m}$ and $\vk = p n \pm O \bc{ \sqrt{ np \log n }   }$ with exponentially high probability by the Chernoff bound for the binomial distribution, the lemma follows.
\end{proof}

\subsection{Thresholding the Score} \label{sec:thresholding}
\paragraph{Outline}
Next, we construct a threshold $T = T(p, m, n)$ with the goal that $\PSI_j - \Erw \brk{ \XI_j } > T$ for all but $\eps \vk$ agents with hidden bit one and  $\PSI_i - \Erw \brk{ \XI_i } < T$ for all but $\eps \vk$ agents with hidden bit one. 
Observe that for $\alpha \in (0,1)$, 
\begin{align*}
    \MoveEqLeft \Delta_j S(0,1) + \alpha \Delta_j ( S(1,1) - S(0,1) ) \\
    & =  \Delta_j S(1,1) - (1 - \alpha) \Delta_j (S(1,1) - S(0,1)).
\end{align*}

Therefore, we optimize the threshold $T$ with respect to $\alpha$ which can be seen as an interpolation between the expected difference in the neighborhood sum for agents of different states.
This gives us the bound
\begin{equation}
\label{bound_m_2}
    m\! >\! \frac{1}{L}\left(\!\ln\left(\frac{1}{p}\right) + 2\ln\left(\frac{2}{\eps\delta}\right) + 2\sqrt{\ln\left(\frac{2}{\eps\delta}\right)\ln\left(\frac{2}{\eps \delta p}\right)}\right) .
\end{equation}

The main theorem follows from \eqref{bound_m_2} and the fact that conditioning on $\cK$ does not harm the result. 
Indeed, by \eqref{bound_m_2}, the algorithm is with probability $1-\delta$ an $\eps$-recovery algorithm. 
The analysis is valid, conditioned on $\cK$. 
As discussed earlier, $\Pr \bc{ \cK } = 1 - n^{- \Omega(1)}$ by Chernoff's bound. 
Consequently, the algorithm only fails to be an $\eps$-recovery algorithm with probability $1- \delta$ on the event $\neg \cK$ which has probability $o_n(1)$.
Therefore, the algorithm is, unconditional, an $\eps$-recovery algorithm with probability at least $(1 - \delta) ( 1 - \Pr \bc{ \cK } )$ which implies \cref{thm_mainresult}. \qed

\paragraph{Formal derivation}

It remains to formally prove \cref{bound_m_2}.
To this end, define 
\[ T_\alpha = \Delta_j S(1,1) - (1 - \alpha) \Delta_j (S(1,1) - S(0,1)) \]
and observe that by \cref{lem_conc_deltastar,lem_chernoff_beta}, we have
\begin{align}
    &\label{eq_healthy} \Pr \bc{ \PSI_j > T_\alpha \mid \SIGMA_j = 0  } \leq 2 \exp \bc{ (- 1 + o(1)) \frac{\alpha^2 (S(1,1) - S(0,1))^2 m }{2 n p_S}  } 
\end{align}

for agents with bit zero and analogously
\begin{align}
    &\label{eq_infected} \Pr \bc{ \PSI_j < T_\alpha \mid \SIGMA_j = 1 }  \leq  2 \exp \bc{  - (1 + o(1)) \frac{(1 - \alpha)^2 (S(1,1) - S(0,1))^2}{2 n p_S} }. 
\end{align}

Let $\tilde \sigma^{(\alpha)}$ be the estimate of $\SIGMA$ by  \cref{algo_mn} using threshold $T_\alpha$.
For brevity, define
\begin{align*}
 L &= L(n, p, S) =\frac{\bc{ S(1,1) - S(0,1) }^2}{ 2 n p_S }, \quad
    N_\alpha = \alpha^2 L m, \quad \text{and}  \quad    P_\alpha = (1-\alpha)^2 L m.
\end{align*}
Define $\cK$ as the event that $\vk \in \bc{ np - \sqrt{ np  } \log n, np + \sqrt{ np  } \log n  }$. 
This means that conditioned on $\cK$ there are roughly as many agents with bit one as we expect. 
Clearly, Chernoff's bound directly implies $\Pr \bc{ \cK } = 1 - n^{- \Omega(1)}$.

We let $\fp(\alpha) = \cbc{x_i : \SIGMA_i = 0, \tilde \sigma^{(\alpha)}_i = 1 }$ be the set of false positives and $\fn(\alpha) = \cbc{x_i : \SIGMA_i = 1, \tilde \sigma^{(\alpha)}_i = 0 } $ the set of false negatives. 
By \eqref{eq_healthy} and \eqref{eq_infected} we find 
\begin{align}
    \label{expectation_error} &\Erw \brk{ \abs{ \fp(\alpha) } \mid \cK } \leq 2 n (1 - p) \exp \bc{  - (1 + o(1)) N_\alpha  } \intertext{ and }
    &\Erw \brk{ \abs{ \fn(\alpha) } \mid \cK } \leq 2 n p \exp \bc{  - (1 + o(1)) P_\alpha }.
\end{align}

Therefore, Markov's inequality yields
\begin{align}
    \label{markov_error} & \Pr \brk{ \abs{ \fp(\alpha) } > \eps n p  \mid \cK } \leq \frac{2 (1 -p) \exp \bc{  - (1 + o(1)) N_\alpha  }}{\eps p}, \\ \notag 
    & \Pr \brk{ \abs{ \fn(\alpha) } > \eps n p \mid \cK } \leq \frac{2 \exp \bc{  - (1 + o(1)) P_\alpha  }}{\eps }.
\end{align}
Thus, it is sufficient to determine the value of $\alpha$ which minimizes the number of queries conducted such that
\begin{align*}
    \frac{2 \exp \bc{-N_\alpha} }{\eps p} < \delta \quad \text{and} \quad \frac{2 \exp \bc{-P_\alpha} }{\eps} < \delta
\end{align*}
or, equivalently,
\begin{align} \label{eq_sufficient_np}
N_\alpha > \log \bc{ \frac{2 }{\eps \delta p} } \quad \text{and} \quad P_\alpha > \log \bc{ \frac{2}{\eps \delta} }.
\end{align}

Because one criterion is increasing in $\alpha$ while the other one is decreasing, the $\alpha$ minimizing both equations can be easily calculated. 
We find as an optimal value for $\alpha$ that
\begin{align*}
    \alpha^\star = \frac{L m - \log \bc{ \frac{2}{\eps \delta} } + \log \bc{ \frac{2}{\eps \delta p} }}{2 L m} = \frac{1}{2} + \frac{\log \bc{\frac{1}{p}}}{2 L m}
\end{align*}
Because $\alpha$ needs to be between zero and one, we directly require
\begin{align} \label{bound_m_1}
    m > \frac{\log \bc{\frac{1}{p}}}{L}
\end{align}
Consequently, the algorithm's threshold reads
\begin{align}
     T_\alpha = \Delta_j S(1,1) - \bc{\frac{1}{2} - \frac{\log \bc{\frac{1}{p}}}{2 L m} } \Delta_j (S(1,1) - S(0,1)). \label{eq_thresh_algo}
\end{align}
Because
\begin{align*}
    N_{\alpha^\star} = (\alpha^\star)^2  L m \quad \text{and} \quad P_{\alpha^\star} &= (1 - \alpha^\star)^2  L m,
\end{align*}
we have
\begin{align*}
    N_{\alpha^\star} = \bc{\frac{1}{2} + \frac{\log \bc{\frac{1}{p}}}{2 L m}}^2 L m
    = \frac{ 1}{4}L m + \frac{1}{2}\log \bc{\frac{1}{p}} + \frac{1}{4 L m}\log \bc{\frac{1}{p}}^2,\\
\end{align*}
and 
\begin{align*}
    P_{\alpha^\star} = \bc{\frac{1}{2} -\frac{\log \bc{\frac{1}{p}}}{2 L m}}^2 L m
    = \frac{ 1}{4}L m - \frac{1}{2}\log \bc{\frac{1}{p}} + \frac{1}{4 L m}\log \bc{\frac{1}{p}}^2.\\
\end{align*}
As a direct consequence, we obtain \cref{bound_m_2} from \eqref{eq_sufficient_np}:
\begin{align*}
    m > \frac{1}{L}\left(\ln(1/p) + 2\ln(2/(\eps\delta)) + 2\sqrt{\ln(2/(\eps\delta))\ln(2/(\eps \delta p))}\right) .
\end{align*}
Observe that \eqref{bound_m_2} does outnumber \eqref{bound_m_1} and is therefore a sufficient condition. 

\section{Discussion and Conclusion}
\label{sec:discussion}
We studied the approximate recovery from pooled data under a very generic noise model and state a performance guarantee for a distributed combinatorial algorithm.
In the noiseless case when the channel-matrix $S$ becomes the identity matrix our bound on the required number of queries in \cref{thm_mainresult} boils down to 
{\dense \[m \sim 2 k \bc{ \log\left(\frac{n}{k}\right) + 2\ln\left(\frac{2}{\eps\delta}\right) + 2\sqrt{\ln\left(\frac{2}{\eps\delta}\right)\ln\left(\frac{2n}{\eps \delta k}\right)} } , \]}where we set $p = k/n$.
In the sublinear regime where the number of agents with a hidden bit one scales as is $k = n^{\theta}$ for some $\theta \in (0,1))$ our result aligns well with the existing literature. 
As previously stated, the algorithm was already analyzed under exact recovery.
More precisely, \citet{gebhard_2022_ipdps, hahnklimroth_2022_icdcs} prove that for exact recovery $ m \sim 2.43 { \bc{1 + \sqrt{\theta}}^2 } k \log(n)$ is sufficient, and their proof reveals that this is also necessary \cite{gebhard_2022_ipdps}. 
In our paper we obtain $ m \sim 2 (1 - \theta) k \log(n) $ for approximate recovery.
The leading constant is much smaller than in the exact recovery task. 
A similar observation holds with respect to the analysis in \cite{hahnklimroth_2022_icdcs} for exact recovery under noise.
Again, the approximate variant decreases the required constant in front of $k \log n$ significantly if $p$ becomes larger, but the order of magnitude is preserved.  

In the empirical analysis we observe three facts that might be surprising at first glance.
First, multi-edges in the pooling graph are harmful to the algorithm's performance.
Second, the degree of regularity has only a vanishing influence on the number of queries required in the noiseless variant.
Third, the degree of regularity becomes relevant if the queries are subject to noise.

Regarding the existence of multi-edges, asymptotic results show that given multi-edges the algorithm should require slightly fewer queries.
More precisely, it is known \cite{gebhard_2022_ipdps,hahnklimroth_2022_icdcs} that a multiplicative factor of $2(1 - \exp(-1/2)) \sim 0.8$ in the number of required queries arises in exact recovery tasks.
A similar calculation could be carried out in our analysis: the factor $\Delta^\star \Delta^{-1} = 1 + o(1)$ appears in our formal analysis as well.  
Of course, this observation is with respect to a solely asymptotic analysis in which, for instance, $\sqrt{m \log m} \ll m$ is required (see \cite{hahnklimroth_2022_icdcs}).
It is not likely that such assumptions hold for realistic instances.
Furthermore, the appearance of the multi-edges reduces the entropy of the test design significantly for smaller values of $n$.
It is known in the related group testing problem that less queries are needed if the system's entropy rises \cite{Noonan_2021}.
This is a plausible explanation for the different behavior on small samples in contrast to asymptotic considerations.

Moreover, in group testing it was also observed that similar thresholding approaches work better on doubly regular designs for small $n$ \cite{tan_2022}.
It is also known that one-sided regular designs and doubly regular designs outperform the simplistic Bernoulli variant substantially \cite{aldridge_2017_a, aldridge_2016}.
In contrast, in the pooled data problem there is no difference in the performance of the different designs from an information-theoretic point of view \cite{feige2020quantitative, gebhard_2022_ipdps}.
This is due to the high density of the pooling graph compared to the one of group testing: all agents are, more or less, exchangeable in the pooled data graphs such that the entropy of the different models is comparable.
Therefore it is well explainable that no strong differences are observed.

Finally, we observed that if noise arises the regularization of the pooling scheme helps significantly.
Indeed, if two agents join differently many queries, their chance to be read falsely at least once is different.
Therefore, in non-regular designs, the single agents are not exchangeable anymore and the system's entropy decreases for small $n$.
This additional variation appearing in non-regular designs seems to be the reason why doubly regular designs perform better.

\paragraph{Summary.}
\label{sec:conclusions}
The proven performance guarantee is uniform in the expected number of agents with hidden bit one.
It was found that substantially fewer queries are required using the same algorithm as necessary for exact recovery.
Guided by recent findings in group testing, we introduced a doubly regular pooling design.
Our formal proof is designed explicitly to this scheme but straight-forward adjustments would yield the same bounds on dense Bernoulli designs or one-sided regular schemes.
Furthermore, our empirical analysis reveals that on small systems, designs that allow agents to join queries multiple times are inferior to corresponding designs in which this is not allowed.
Moreover, doubly regular designs outperform non-regular designs if the queries are subjected to noise but do not improve upon the number of queries required in the noiseless variant.

One obvious follow-up question is whether the agents' bits could be learned approximately even better with the use of neural networks instead of plain combinatorial constructions.
A second open question relates to an extension of the model: in the pooled data problem we assume a query to report the exact number of agents with bit one (up to noise). In \emph{semi-quantitative group testing}, much less information is given: a query outputs one specific value for a whole range of inputs. It is exciting future research whether a similar approach can be carried out under this model.

\bibliographystyle{abbrvnat}
\bibliography{bibliography}

\end{document}